\documentclass[10pt,letterpaper,peerreview]{IEEEtran}

\usepackage[pdftex]{graphicx}
\DeclareGraphicsExtensions{.pdf}
\usepackage{amsmath}
\usepackage[ruled,linesnumbered]{algorithm2e}

\usepackage{times}
\usepackage{amssymb}

\usepackage{amsthm}

\usepackage{pbox}

\newcommand{\norm}[1]{\left\lVert #1 \right\rVert}

\begin{document}
%
\title{Color Homography}
\author{
\IEEEauthorblockN{Graham D. Finlayson\IEEEauthorrefmark{1}, Han Gong\IEEEauthorrefmark{1}, and Robert B. Fisher\IEEEauthorrefmark{2}}\\
\IEEEauthorblockA{\IEEEauthorrefmark{1}School of Computing Sciences, University of East Anglia, UK}
\\
\IEEEauthorblockA{\IEEEauthorrefmark{2}School of Informatics, University of Edinburgh, UK}
}


\maketitle
\IEEEpeerreviewmaketitle

\section{Introduction}
Homographies are at the center of geometric methods in computer vision and are used in geometric camera calibration, 3D reconstruction, stereo vision and image mosaicking among other tasks. In this paper, we show the surprising result that colors across a change in viewing condition (changing light color, shading and camera) are also related by a homography.

\begin{figure}
\begin{center}
  \includegraphics[width=\linewidth]{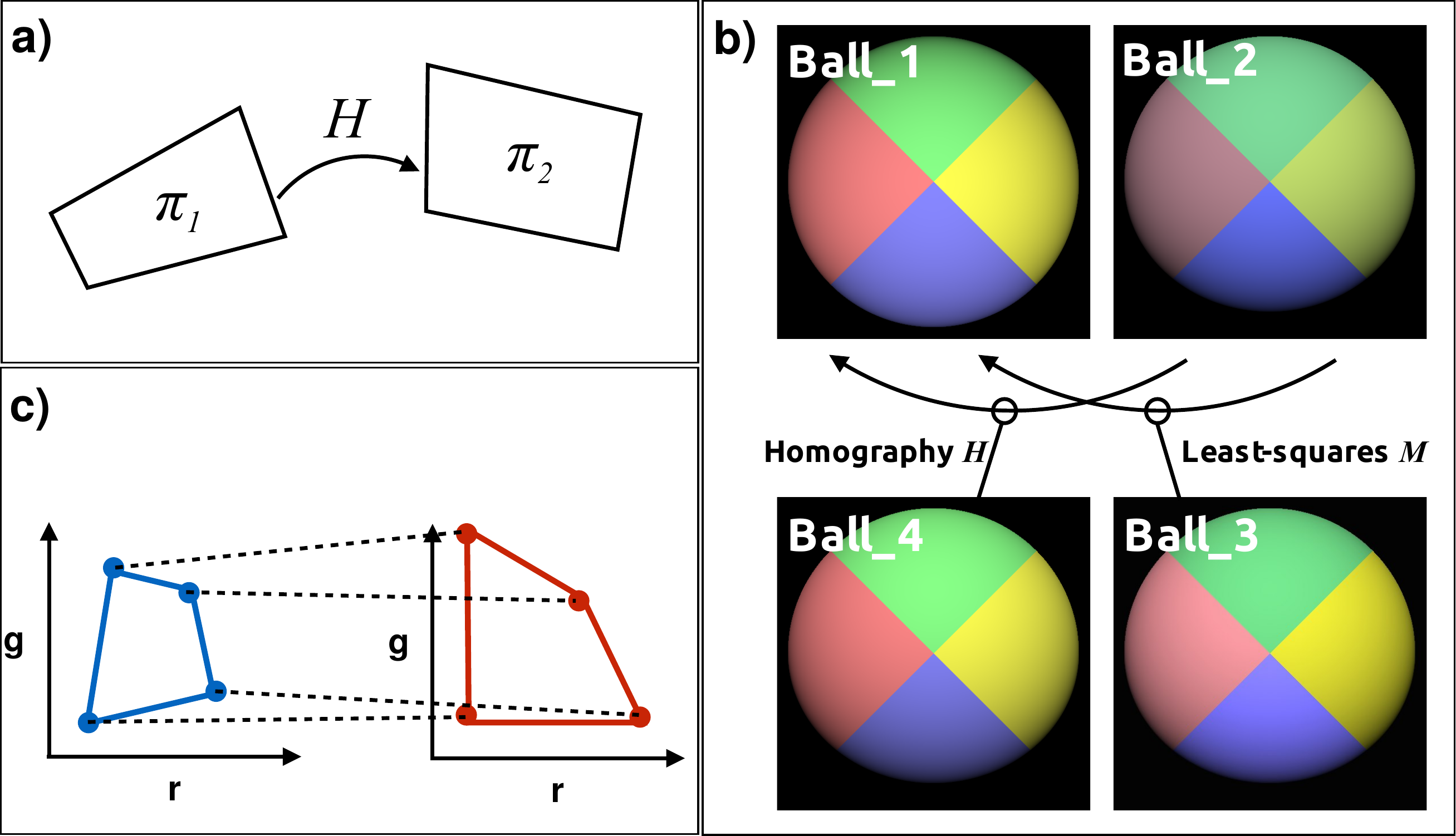}
\end{center}
\caption{Top left, panel (a), images of two planes are related by an homography. Right, panel (b), 4 images of a colored ball are shown. Ball\_1 is the reference image where the illumination color is white and placed behind the camera. Ball\_2 is the object illuminated with a blue light from above. Respectively, Ball\_3 and Ball\_4 are the least-squares mapping and the homography match from Ball\_2 to Ball\_1, Bottom right, panel (c), the chromaticities from Ball\_2 matched to corresponding chromaticities in Ball\_1.}
\label{fig:balls}
\end{figure}

In geometry computer vision, an homography relates two planes. In Figure~\ref{fig:balls}a, $\pi_1$ and $\pi_2$ might denote the same 3D plane viewed in two different images related by the homography $H$.
In color, an homography relates two {\it photometric} views. In Figure~\ref{fig:balls}b, Ball\_1 is the image of the side-view of a 4-color ball where the ball is lit from behind the camera with a white light. The same ball is lit from above with a bluish light, image Ball\_2. The images are in pixel-wise correspondence. We first color correct Ball\_2 to match Ball\_1 by using linear regression. Its results is shown in image Ball\_3 where the colors are incorrectly mapped and the red color segment looks particularly wrong. Ball\_4 is a better example that Ball\_2 is color corrected by using the linear transform by color homography.
In this paper, we propose that to map one {\it photometric} view to another we must map the colors correctly independent of shading. Since shading only affects the brightness, or magnitude, of the RGB vectors, it is possible to find the $3\times3$ map which maps the color {\it rays} (the RGBs with arbitrary scalings) in one photometric view to corresponding rays in another. We note that this ''ray matching'' is precisely the circumstance in geometric mapping for co-planar and corresponding points in two images~\cite{Hartley2004}. 
An RGB without shading can be encoded as the rg-chromaticity coordinate. In Figure~\ref{fig:balls}c, the 4 reflectances from the ball correspond to 4 points in an rg-chromaticity diagram which define the quadrilaterals shown in the left and right of the panel (for respectively for the images Ball\_2 and Ball\_1). The mapping between the two chromaticity diagrams is an homography. We show that the color calibration problem -- mapping device RGBs recorded for a color chart to corresponding XYZs -- can be formulated as a homography problem. Because in real shading varies across the chart solving for the homography can deliver a 50\% improvement in color correction compared with direct linear least-squares regression.

\section{Background}
\label{sec:bg}

For the geometric planar homography problem, we write:
\begin{equation}
\left [
\begin{array}{c}
\alpha x\\
\alpha y\\
\alpha
\end{array}
\right ]^\intercal
=
\left [
\begin{array}{c}
x'\\
y'\\
1
\end{array}
\right ]^\intercal
\left [
\begin{array}{ccc}
h_{11} & h_{12} & h_{13} \\
h_{21} & h_{22} & h_{23} \\
h_{31} & h_{32} & h_{33} 
\end{array}
\right ]
\;\triangleq\;
\underline{x}=H(\underline{x}')
\label{eq:homography}
\end{equation}
In Equation~\ref{eq:homography}, $(x,y)$ and $(x',y')$ denote corresponding image points -- the same physical feature -- in two images. In homogeneous coordinates the vector [$a\; b \;c]^\intercal$ maps to the coordinates $[a/c \;b/c]^\intercal$ and so, in Equation~\ref{eq:homography}, the scalar $\alpha$ cancels to form the image coordinate $(x,y)$. For all pairs of corresponding points $(x,y)$ and $(x',y')$ that lie on the same plane in 3 dimensional space, Equation~\ref{eq:homography} exactly characterises the relationship between their images~\cite{Hartley2004}.
To solve for an homography (e.g. for image mosaicking), we need to find at least 4 corresponding points in a pair of images.


\section{Color Homography}
\label{sec:main}

Let us map an RGB $\underline{\rho}$ to a corresponding RGI (red-green-intensity) \underline{c} using a full-rank $3\times 3$ matrix $C$:
\begin{equation}
\begin{array}{c}
\underline{\rho}^\intercal C=\underline{c}^\intercal\\
\;\\
\left [
\begin{array}{c}
R\\
G\\
B
\end{array}
\right ]^\intercal
\left [
\begin{array}{ccc}
1 & 0 & 1\\
0 & 1 & 1\\
0 & 0 & 1
\end{array}
\right ]
=
\left [ 
\begin{array}{c}
R\\
G\\
R+G+B
\end{array}
\right ]^\intercal

\end{array}
\label{eq:chromaticity_conversion}
\end{equation}
The $r$ and $g$ chromaticity coordinates are written as ${r=R/(R+G+B)} \;,\;{g=G/(R+G+B)}$ interpreting the right-hand-side of Equation~\ref{eq:chromaticity_conversion} as a homogeneous coordinate we see that $
\underline{c}\propto \left [
\begin{array}{ccc} r&g&1
\end{array}
\right]^\intercal
$. In the following proof it is useful to represent 2-d chromaticities by their corresponding 3-d homogeneous coordinates.

\newtheorem{thm}{Theorem}
\begin{thm}[Colour Homography]
Chromaticities across a change in capture condition (light color, shading and imaging device) are a homography apart.
\end{thm}

\begin{proof}
    First we assume that across a change in illumination or a change in device where the shading is the same (for the Mondrian-world) the corresponding RGBs are related by a linear transform M. Clearly, $H=C^{-1}MC$ maps colors in RGI form between illuminants. Due to different shading, the RGI triple under a second light is represented as $\underline{c}'^\intercal=\alpha'\underline{c}^\intercal H$, where $\alpha'$ denotes the unknown scaling. Without loss of generality let us interpret \underline{c} as a homogeneous coordinate i.e. assume its third component is 1. Then, $[r'\;g']^\intercal=H([r\;g]^\intercal)$ (chromaticity coordinates are a homography $H()$ apart).
\end{proof}


\subsection{Solving color Homography by Alternating Least Squares}
\label{sec:main3d}

Suppose $A$ and $B$ denote respectively $n\times 3$ matrices of $n$ corresponding pixels with respect to two images of the same scene where the illumination changes (and, also possibly the camera properties). The color change is modeled as a linear transform. And, due to the relative positions of light and surfaces, the per-pixel shading intensities are usually different. Assuming the Lambertian image formation, $ DAH\approx B$ where $D$ is an $n\times n$ diagonal matrix of shading factors and $H$ is a $3\times 3$ color correction matrix. We solve this equation by using Alternating Least-Squares (ALS) described in Algorithm~\ref{alg:als}.
\begin{algorithm}
\SetAlgoLined
    $i=0$, $A^0=A$\;
    \Repeat{$\norm{A^{i}-A^{i-1}}<\epsilon$}{
    $i=i+1$\;
    $\min_{D^i} \norm{D^iA^{i-1}-B}$\;
    $\min_{H^i} \norm{D^iA^{i-1}H^i-B}$\;
    $A^{i}=D^iAH^i$\;
    }
\caption{Homography from alternating least-square}
\label{alg:als}
\end{algorithm}
The effect of the individual $H^i$ and $D^i$ can be combined into a single matrix $D=\prod_{i} D^i$ and $H=\prod_{i} H_i$.

Alternately, we can solve for $H$ in direct analogy to computer vision. Specifically, we randomly select sets of 4 corresponding points and find the best $H$ and then test the ‘goodness’ of fit for the $H$ discovered on the rest of the dataset. This RANSAC method (Random sampling consensus) has the advantage that we can robustly minimize any chosen error including CIE Lab.

Finally, we note that existing works have already developed intensity invariant color matching ~\cite{ALS,FuntIntensity}. Though, neither has reported the homography formalism or exploded the enhanced flexibility it offers.
\section{Color Correction}
\label{sec:exp}

\begin{figure}
\begin{center}
  \includegraphics[width=\linewidth]{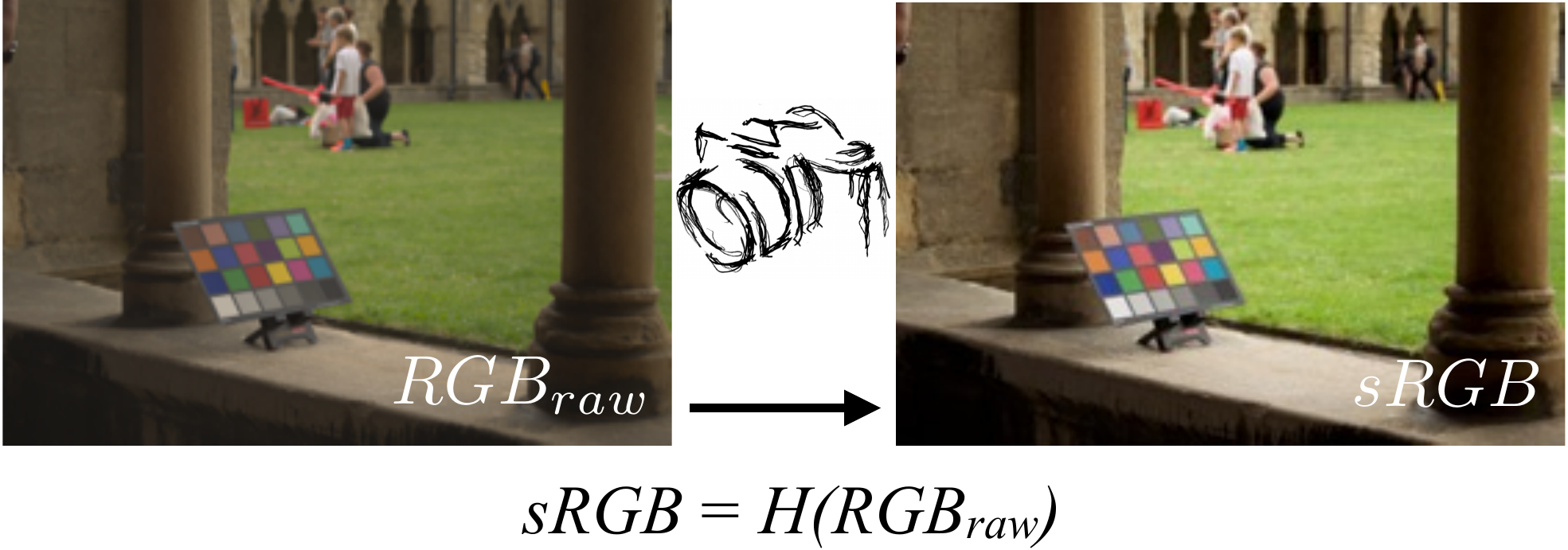}
\end{center}
\caption{Color correction (mapping raw to display sRGB) is an homography problem. This figure is also a color chart example which was used for our color correction evaluation.}
\label{fig:chart}
\end{figure}
In color correction - mapping raw RGBs to a display color space - the target RGBs, are known to vary in intensity. In Figure~\ref{fig:chart}, we show the picture of an image in the raw RGB space of a camera and the corresponding reproduction when the colors are corrected for display. Very professional imaging scientist might take a picture of a color checker and a second picture of a uniform gray target with same size in the same location. By dividing the RGB image of the checker by the image of the gray-target the shading is removed and then the shading corrected RGBs can be mapped to known reference display color coordinates using simple least-squares. However, this two-step approach is inconvient and in some cases cannot be done at all (e.g. in an on-going surveillance situation). In our evaluation, this two-step approach is applied to obtain ground truth RGB inputs with corrected shading. For each method, we obtain a $3\times 3$ color correction matrix from the non-shading-corrected RGB inputs to a reference target and then apply this matrix to the shading corrected RGB inputs. This experimental methodology is described in detail in \cite{FuntIntensity}.

The color correction error is measured using the CIE Lab metric~\cite{WYSZECKI82}. All 3 of our evaluation images were taken around a local historical site that is popular with amateur and professional photographers alike (e.g. Figure~\ref{fig:chart}).
The mean, median, 95\% quantile and max $\Delta$E errors are reported in Table~\ref{tab:calibration_error}. It is clear that homography-based color correction supports a significantly improved color correction performance (all error measures are $\approx$ 40\% improved).
\begin{table}
\begin{center}
\begin{tabular}{|l|c|c|c|c|}
\hline
Method & mean & median & 95\% & max \\
\hline
Least-squares & 3.70 & 3.30 & 7.73  & 8.39 \\
Homography & 2.34 & 2.09 & 5.02 & 5.43 \\
\hline
\end{tabular}
\end{center}
\caption{CIE $\Delta$E error comparison}
\label{tab:calibration_error}
\end{table}


\section{Conclusion}
\label{sec:con}
In this paper, we shown the surprising result that colors across a change in viewing condition (changing light color, shading and camera) are related by a homography. Our homography color correction application delivers improved color fidelity compared with the linear least-squares.
{\small
\bibliographystyle{IEEEtran}
\bibliography{biblio.bib}

\begin{thebibliography}{1}
\providecommand{\url}[1]{#1}
\csname url@samestyle\endcsname
\providecommand{\newblock}{\relax}
\providecommand{\bibinfo}[2]{#2}
\providecommand{\BIBentrySTDinterwordspacing}{\spaceskip=0pt\relax}
\providecommand{\BIBentryALTinterwordstretchfactor}{4}
\providecommand{\BIBentryALTinterwordspacing}{\spaceskip=\fontdimen2\font plus
\BIBentryALTinterwordstretchfactor\fontdimen3\font minus
  \fontdimen4\font\relax}
\providecommand{\BIBforeignlanguage}[2]{{%
\expandafter\ifx\csname l@#1\endcsname\relax
\typeout{** WARNING: IEEEtran.bst: No hyphenation pattern has been}%
\typeout{** loaded for the language `#1'. Using the pattern for}%
\typeout{** the default language instead.}%
\else
\language=\csname l@#1\endcsname
\fi
#2}}
\providecommand{\BIBdecl}{\relax}
\BIBdecl

\bibitem{Hartley2004}
R.~I. Hartley and A.~Zisserman, \emph{Multiple View Geometry in Computer
  Vision}, 2nd~ed.\hskip 1em plus 0.5em minus 0.4em\relax Cambridge University
  Press, ISBN: 0521540518, 2004.

\bibitem{ALS}
G.~D. Finlayson, M.~Mohammadzadeh~Darrodi, and M.~Mackiewicz, ``The alternating
  least squares technique for nonuniform intensity color correction,''
  \emph{Color Research \& Application}, vol.~40, no.~3, pp. 232--242, 2015.

\bibitem{FuntIntensity}
B.~Funt and P.~Bastani, ``Irradiance-independent camera color calibration,''
  \emph{Color Research \& Application}, vol.~39, no.~6, pp. 540--548, 2014.

\bibitem{WYSZECKI82}
G.~Wyszecki and W.~Stiles, \emph{Color Science: Concepts and Methods,
  Quantitative Data and Formulas}, 2nd~ed.\hskip 1em plus 0.5em minus
  0.4em\relax Wiley, New York, 1982.

\end{thebibliography}
}

\end{document}